\def\year{2020}\relax
\documentclass[letterpaper]{article} 
\usepackage{aaai20}  
\usepackage{times}  
\usepackage{helvet} 
\usepackage{courier}  
\usepackage[hyphens]{url}  
\usepackage{graphicx} 
\usepackage{graphicx}  
\newtheorem{proof}{Proof}
\newtheorem{prop}{Proposition}

\usepackage{amsmath}
\usepackage{amssymb}
\usepackage{algorithm}
\usepackage{algorithmic}
\usepackage{array}
\usepackage{multirow}

\title{Dynamical System Inspired Adaptive Time Stepping Controller for \\ Residual Network Families}
\author{Yibo Yang\textsuperscript{\rm 1,2,}\thanks{Equal Contribution},
	Jianlong Wu\textsuperscript{\rm 2,3,$*$}, 
	Hongyang Li\textsuperscript{\rm 2},
	Xia Li\textsuperscript{\rm 2,4},
	Tiancheng Shen\textsuperscript{\rm 1,2},
	Zhouchen Lin\textsuperscript{\rm 2,5,}\thanks{Corresponding Author}\\ 
	\textsuperscript{\rm 1}Center for Data Science, Academy for Advanced Interdisciplinary Studies, Peking University\\
	\textsuperscript{\rm 2}Key Laboratory of Machine Perception (MOE), School of EECS, Peking University\\
	\textsuperscript{\rm 3}School of Computer Science and Technology, Shandong University\\
	\textsuperscript{\rm 4}Key Laboratory of Machine Perception, Shenzhen Graduate School, Peking University\\
	\textsuperscript{\rm 5}Samsung Research China – Beijing (SRC-B)\\
	{\tt \{ibo,jlwu1992,lhy\_ustb,ethanlee,tianchengshen,zlin\}@pku.edu.cn} 
}

\begin{document}

\maketitle

\begin{abstract}
The correspondence between residual networks and dynamical systems motivates researchers to unravel the physics of ResNets with well-developed tools in numeral methods of ODE systems. The Runge-Kutta-Fehlberg method is an adaptive time stepping that renders a good trade-off between the stability and efficiency. Can we also have an adaptive time stepping for ResNets to ensure both stability and performance? In this study, we analyze the effects of time stepping on the Euler method and ResNets. We establish a stability condition for ResNets with step sizes and weight parameters, and point out the effects of step sizes on the stability and performance. Inspired by our analyses, we develop an adaptive time stepping controller that is dependent on the parameters of the current step, and aware of previous steps. The controller is jointly optimized with the network training so that variable step sizes and evolution time can be adaptively adjusted. We conduct experiments on ImageNet and CIFAR to demonstrate the effectiveness. It is shown that our proposed method is able to improve both stability and accuracy without introducing additional overhead in inference phase.
\end{abstract}

\section{Introduction}

Currently, the structure of neural network is mainly developed by hand-crafted design \cite{simonyan2014very,szegedy2015going} or neural architecture searching \cite{baker2016designing}. A theoretical guidance is still lacking for understanding deep network behaviors. One of the most successful architectures, residual network (ResNet) \cite{he2016deep}, introduces identity mappings to enable training a very deep model. ResNets are also used as the base model for a series of computer vision tasks such as scene segmentation \cite{chen2017deeplab}, and action recognition \cite{tran2018closer}. Despite the huge success, the understanding of ResNets is mainly supported by empirical analyses and experimental evidences, other than some attempts from an optimization view \cite{li2018optimization}. Recently, the connection between ResNet and dynamical system has inspired researchers to unravel the physics of residual networks using the rich theories and techniques in differential equations \cite{weinan2017proposal,haber2017stable}. 


In \cite{weinan2017proposal}, it is noted that the Euler method for ordinary differential equations (ODEs) has the same formulation as ResNet iterative updates, and ResNet is viewed as a discrete dynamical system. In this way, the parameter learning in neural networks is translated into its continuous counterpart as an optimal control problem \cite{chen2018neural,li2018icml,li2017maximum,behrmann2018invertible}. Based on the Euler method, some studies introduce multi-step or higher-oder discretization \cite{Lu2018,he2019ode}, and fractional optimal control \cite{jia2019focnet} to construct more powerful network structures for different tasks. Other studies analyze the stability of residual networks and propose more stable and robust structures \cite{haber2017stable,chang2018reversible,ruthotto2018deep,eldad2019}.

An appropriate time stepping for discretization methods of ODEs is crucial for the stability and efficiency \cite{ascher1998computer}. A small step size is able to render an accurate solution, but requires more steps for a fixed evolution time. Chang \emph{et al.} adopt a multi-level strategy to adjust the time step size for ResNet training \cite{chang2017multi}. In a recent study, a small step size is suggested for more stable and robust ResNets \cite{zhang2019towards}. However, an overly small step size would smooth the feature learning. From a dynamical system view, the evolution time for network with a fixed depth and a small step size is too short for the system to evolve from the initial state to the final linearly separable state. In numerical methods for ODEs, adaptive time stepping strategies, such as the Runge-Kutta-Fehlberg (RKF) method \cite{hairer1991solving} as shown in Figure 1, are able to attain a good trade-off between stability and cost. Can we also design an adaptive time stepping for ResNets to ensure both stability and performance?

In this study, we analyze the effects of time stepping on the stability and performance of residual networks, and point out that each step size should be aware of previous steps and the weight parameters in the current step. We develop an adaptive controller, which connects different steps as an LSTM, and takes the parameters of each step as input, to output a set of coefficients that decide the current step size. In doing so, the network is trained with variable step sizes and evolution time, so that our time stepping is optimized jointly to render the network a better stability and performance. More importantly, because the controller is data-independent, our performance gains come with no additional cost in inference phase.

\begin{figure}[t]
	\begin{center}
		\includegraphics[width=1.0\linewidth]{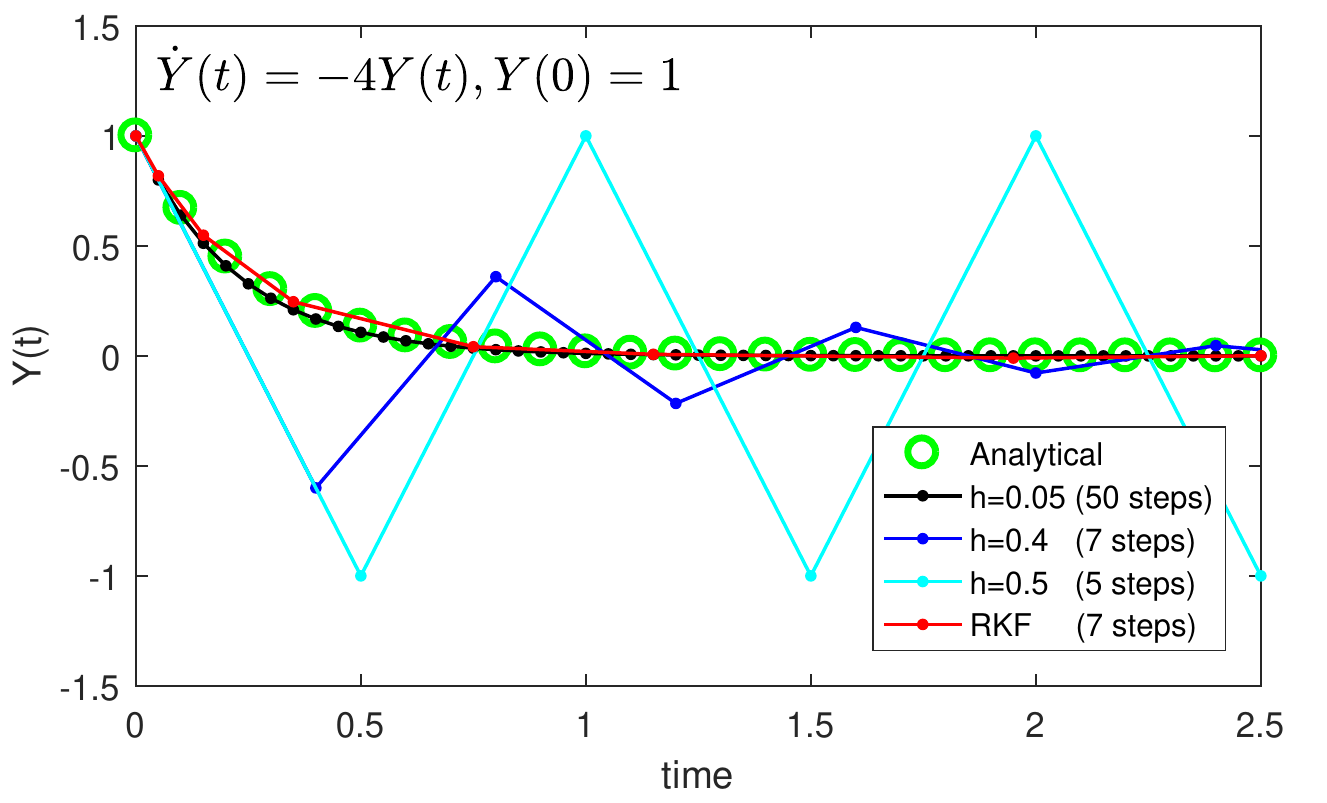}
	\end{center}
	\caption{An initial value problem with different discretization step sizes $h$. The RKF method is able to attain a good trade-off between stability and efficiency, offering a stable solution, while using a small number of steps.}
	\label{fig1}
\end{figure}

The contributions of this study can be listed as follows:
\begin{itemize}
	\item We analyze the correspondence between ODEs and ResNets, establish a stability condition for ResNets with step sizes and weight parameters, and point out the effects of step sizes on the stability and performance. 
	\item Based on our analyses, we develop a self-adaptive time stepping controller to enable optimizing variable step sizes and evolution time jointly with the network training. 
	\item Experiments on ImageNet and CIFAR demonstrate that our method is able to improve both stability and accuracy. The improvements come with no additional cost in inference phase. We also test the application of our method to two non-residual network structures.
\end{itemize}

\section{View ResNet as a Dynamical System}
The forward propagation in a residual network block \cite{he2016deep} can be written as:
\begin{equation}
\mathbf{y}_{j+1} = \mathbf{y}_{j} + h\mathcal{F}(\mathbf{y}_{j}, \mathbf{W}_{j}),\quad j=0,1,...,D-1,
\label{res}
\end{equation}
where $\mathcal{F}$ is the residual function for each step, and $D$ is the network depth. Here we add the $h>0$ in a multiplicative way with the residual branch. Usually a unit of $\mathcal{F}$ has the form of $\mathcal{F}=\sigma(BN(\mathbf{W}_j \mathbf{y}_j))$, where $\sigma$ is the non-linear activation. 
When $h=1$, it reduces to the original form of ResNets. Regarding $h$ as a fixed step size, we see that Eq. (\ref{res}) can be interpreted as the forward Euler method discretization for the following initial value problem (IVP) \cite{weinan2017proposal,haber2017stable}:
\begin{equation}
\mathbf{\dot{y}}(t) = \mathcal{F}(\mathbf{y}(t), \mathbf{W}(t)),\quad \mathbf{y}(0) = \mathbf{y}_0,
\label{ode}
\end{equation}
where features $\mathbf{y}(t)$ and parameters $\mathbf{W}(t)$ are viewed in their continuous limit as a function of time $t\in[0,T]$. The evolution time $T$ corresponds to the network depth $D$. In doing so, residual networks are interpreted as the discrete counterpart of dynamical systems, and parameter learning is equivalent to solving an optimal control problem with respect to the ODE system in Eq. (\ref{ode}) \cite{chen2018neural,li2018icml,haber2018learning}. Related studies use the stability condition of the forward Euler method to analyze the stability of ResNets and propose better structures \cite{chang2018reversible,eldad2019}. We show that time stepping is crucial for the stability and performance of ResNets.

\subsection{Time Stepping for the Euler method}

Given a linear problem $\dot{y}(t)=\lambda y(t),t\in[0,T]$, we have its forward Euler's discretization as $y_{j+1}=y_{j} + h\lambda y_j,j=0,...,N-1$, where $h$ is the fixed step size and $T=Nh$. Assuming that $y_0^{\epsilon}=y_0 + \epsilon$ is the initial value suffered from a perturbation $\epsilon>0$, we have:
\begin{equation}
|y_N^{\epsilon}-y_N|=|1+h\lambda|^N\epsilon,
\end{equation}
which indicates that when $N\rightarrow\infty$, the perturbation is controllable if $|1+h\lambda|\le1$. As a more general case, the forward Euler method for non-linear system Eq. (\ref{ode}) is stable when the following condition holds \cite{ascher1998computer}:
\begin{equation}
\max |1+h\lambda_i(\mathbf{J}(t))|\le 1,
\label{jacob}
\end{equation}
where $\lambda_i$ denotes the $i$-th eigenvalue of the Jacobian matrix defined as $\mathbf{J}(t)=\nabla_{\mathbf{y}} \mathcal{F}(\mathbf{y}(t), \mathbf{W}(t))$. From the stability condition, we can see that a small step size $h$ is required to obtain a stable solution. In practical implementations, the step size $h$ should satisfy a stable solution, while being as large as possible to reduce the amount of iterative steps for a fixed evolution time $T$. Thus, the choice for a time stepping scheme is crucial for both stability and efficiency. 

The Runge-Kutta-Fehlberg (RKF) method \cite{hairer1991solving} as an adaptive time stepping is able to attain a good trade-off between the stability and efficiency. It uses the $p$-th (usually $p$=4) order Runge-Kutta method to compute the current solution $y_{j+1}$, and the local truncation error is:
\begin{equation}
|y_{j+1}-y(t_{j+1})|=O(\Delta t^{p+1}_j),
\end{equation}
where $\Delta t_j$ is the current step size. The $y(t_{j+1})$ is approximated by the $p+1$-th order form, denoted as $\hat{y}_{j+1}$. Then the new step size can be adjusted as:
\begin{equation}
\Delta t_{j+1}=k\times\Delta t_j\times\left(\frac{Tol}{|y_{j+1}-\hat{y}_{j+1}|}\right)^{1/(p+1)},
\end{equation}
where $k$ is a factor and $Tol$ is a tolerance error. The method adaptively increases or reduces the next step size according to the agreement between $y_{j+1}$ and $\hat{y}_{j+1}$. 

As a simple example, we consider the problem , $\dot{y}(t)=-4y(t), y(0)=1$, whose analytical solution can be easily derived as $y(t)=\exp(-4t)$. As shown in Figure \ref{fig1}, the adaptive method RKF is able to offer a stable solution, but requires significantly less number of steps than a small step size for the evolution period. If we use a large step size to reduce the number of steps, the solution will be unstable. Thus, an adaptive time stepping scheme is crucial for the stability and efficiency of solution to ODE systems.

\subsection{Time Stepping for ResNets}

As a discrete counterpart of dynamical systems, ResNet has similar behaviors to the discretization method for ODEs. We show that time stepping causes similar effects on the stability and performance of ResNets. 

\begin{prop}
	Consider a ResNet with $D$ residual blocks and variable step sizes $\Delta t_j$ for each step. Let $\epsilon$ be the perturbation coming from noise or adversary and satisfies $||\mathbf{y}^{\epsilon}_0-\mathbf{y}_0||=\epsilon$. We have:
	\begin{equation}
	||\mathbf{y}^{\epsilon}_D-\mathbf{y}_D||\le \epsilon \cdot \prod_{j=0}^{D-1}(1+||\mathbf{W}_j||_2\Delta t_j),
	\label{res_stab}
	\end{equation}
	where $||\mathbf{W}_j||_2$ denotes the spectral norm of weight matrix in each residual block.
\end{prop}

\begin{proof}
	See Appendix A for its proof.$\hfill\blacksquare$  
\end{proof}

We note that the robustness of ResNets to perturbation is affected by the network depth, spectral norm of each weight matrix, and each step size. Eq. (\ref{res_stab}) shows that the stability of ResNets is conditioned on each layer in a stacking way, which is consistent with previous findings \cite{veit2016residual}. For each layer, it is suggested that the weight matrix should have a small spectral norm. Since $||\mathbf{W}_j||_2^2\le||\mathbf{W}_j||_F^2$, weight decay that regularizes the Frobenius norm is effective to train models with robustness to input perturbations. However, it shrinks the weight matrix in all directions, and discards information of input features. Some studies propose spectral norm regularizer \cite{yoshida2017spectral,miyato2018spectral} or Jacobian regularizer \cite{sokolic2017robust} to improve the stability. 

The term in Eq. (\ref{res_stab}) has a similar form to Eq. (\ref{jacob}), and also indicates that a small step size $\Delta t_j$ should be chosen for ResNets' stability. Nevertheless, as pointed out by \cite{zhang2019towards}, an overly small step size would smooth the feature learning. Denoting loss function as $L$, in ResNets with variable step sizes $\Delta t$, we have the gradient backpropagated to layer $\mathbf{y}_n$ as:
\begin{equation}
\frac{\partial L}{\partial \mathbf{y}_n} = \frac{\partial L}{\partial \mathbf{y}_D}\left[\mathbf{1}+\frac{\partial}{\partial\mathbf{y}_n} \sum_{i=n}^{D-1}\mathcal{F}(\mathbf{y}_i, \mathbf{W}_i)\Delta t_i\right],
\end{equation}
which shows that the backpropagated information for each layer comes from two terms. When the step size $\Delta t_i$ is too small, the second term would vanish, and gradients for each layer would be the same as $\frac{\partial L}{\partial \mathbf{y}_D}$. This would make the network inefficient and lacking in representation power. From the dynamical system perspective, if the network with a fixed depth has a small step size, its corresponding optimal control problem would have a short period of evolution time, which increases the difficulty of transforming the data space from the initial state to the expected linearly separable state. 

Therefore, similar to the discretization for ODE systems, ResNets also need an adaptive time stepping to enable a good trade-off between the stability and performance. A related study \cite{zhang2018smooth} proposes to optimize step sizes $\Delta t_j$ as explicit parameters. 
We note that independent step sizes are not self-aware and cannot be adjusted adaptively. Inspired by our analyses, we propose a self-adaptive time stepping controller that is dependent on the weight matrices and aware of previous steps.



\section{Proposed Methods}
In this section, we first introduce our design in the optimal control view, and then describe the components of our self-adaptive time stepping controller. Finally, we analyze the complexity of our method and show implementation details. 

\subsection{The Optimal Control View}
In our analyses, we note that the product of step size and spectral norm of weight matrix in each layer decides the stability. Directly calculating the spectral norm requires the SVD decomposition, which makes the training inefficient. Here we introduce a controller that outputs the current step size $\Delta t_j$ dependent on the convolution parameters $\mathbf{w}_j$ in this layer. Besides, similar to the design of RKF method, the new step size should remember previous step sizes to avoid sharp increment or reduction. In line with these views, denoting the controller as $\Theta$ parameterized by $\{\theta_d\}^{D-1}_{d=1}$, we have the corresponding optimal control problem as:
\begin{align}
\min_{\mathbf{w},\theta}\quad & J=\frac{1}{S}\sum^S_{s=1}\Phi(\mathbf{y}_s(T), \mathbf{y}^*_s)+\sum^{D-1}_{d=1}R(\mathbf{w}(t_d),\theta(t_d)),\label{opti}\\
s.t. \quad & \mathbf{y}_s(t_{d+1}) =\mathbf{y}_s(t_{d})+\mathcal{F}\left(\mathbf{y}_s(t_{d}), \mathbf{w}(t_d)\right)\Delta t_d \notag\\
& \Delta t_d =\Theta\left(\mathbf{w}(t_d);\Delta t_1,...,\Delta t_{d-1};\theta(t_d)\right) \notag\\
& t_0=0,\ t_{d+1}=t_d + \Delta t_d,\ \mathbf{y}_s(0)=\mathbf{y}_s\notag\\
& T=t_D=\sum^{D-1}_{d=0}\Delta t_d,\ d=0,1,...,D-1\notag
\end{align}
where $\Phi$ is the loss function, $R$ is the regularization, $\mathbf{y}^*_s$ is the label of input image $\mathbf{y}_s$, and $S$ is number of samples. The optimal control problem in discrete time \cite{kwakernaak1972linear} looks for the best control parameters $\{\mathbf{w},\theta\}$ for this dynamical system that aims to minimize the cost $J$. From Eq. (\ref{opti}) we can see that, the system has variable step sizes $\Delta t_j$ and evolution time $T$. In implementations, the controller is jointly optimized with the network training, so that an optimal time stepping can be searched to render the network better stability and performance.


\subsection{Self-adaptive Time Stepping Controller}

\begin{figure}[t]
	\centering
	\includegraphics[width=1.0\linewidth]{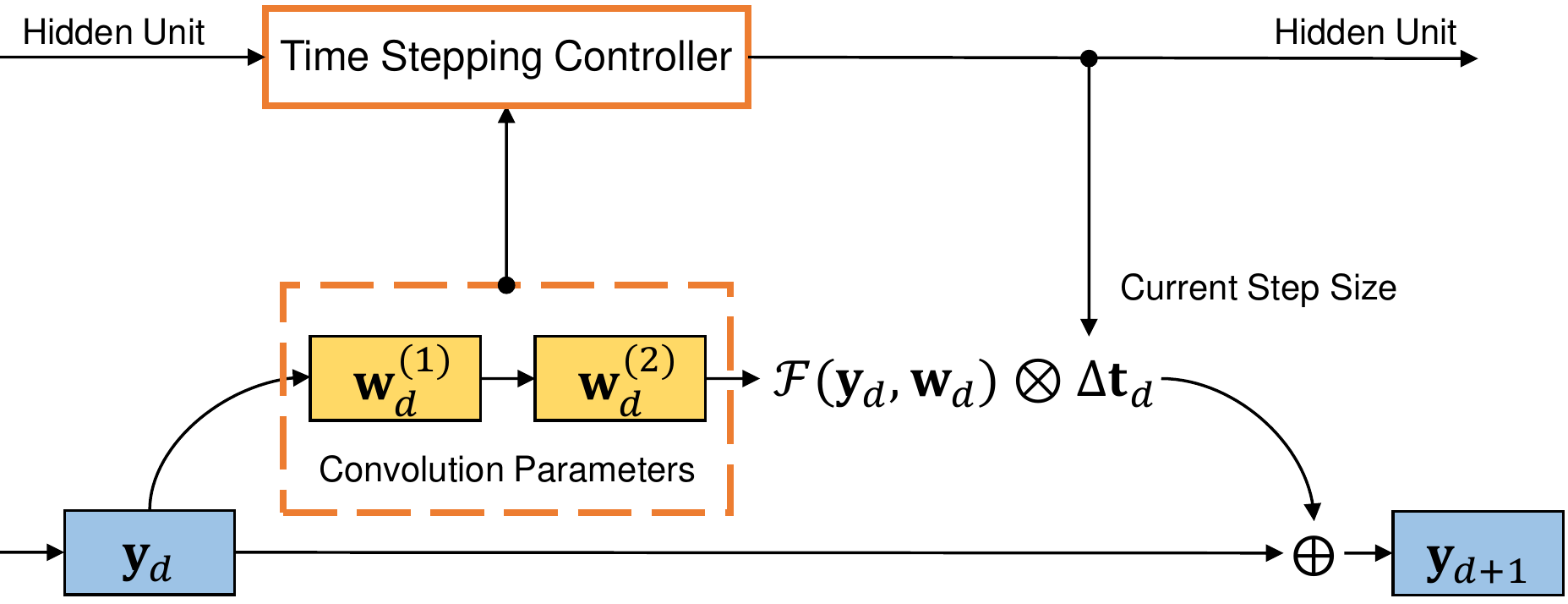}
	\caption{An illustration of our proposed time stepping controller, which is dependent on the convolution parameters, and connects different steps to be aware of previous steps. }
	\label{fig2}
\end{figure}

Since the time stepping controller takes the convolution parameters as input and is aware of previous steps, we parametrize the controller as an LSTM that connects different steps. In implementations, we split the step size as a vector $\Delta \mathbf{t}$ with the same channel number as the feature in current step. The product between residual branch and step size is replaced with a channel-wise multiplication. We find this helps to improve the training stability and accuracy. 

An illustration of our method is shown in Figure \ref{fig2}. Denote the convolution parameters of the $d$-th layer as $\mathbf{w}_d\in\mathbb{R}^{k_1\times k_2\times C_1 \times C_2}$, where $k_1$, $k_2$ are the kernel sizes, and $C_1$, $C_2$ are the number of channels for the input and output, respectively. In order to acquire representative information of the residual function, we average $\mathbf{w}_d$ by projecting along the input dimension and get $\mathbf{\bar{w}}_d\in\mathbb{R}^{k_1k_2C_2}$ after reshaping. We concatenate them if there are multiple convolution layers in the residual branch. The $\mathbf{\bar{w}}_d$ first goes through a transformation layer to reduce the dimension into a reduction of $r$ of the channels in the current layer:
\begin{equation}
\mathbf{x}_d=g_{in}(W_{in}\cdot\mathbf{\bar{w}}_d+b_{in})\in\mathbb{R}^{\frac{C_2}{r}},
\end{equation}
where $g_{in}$ refers to the ReLU function, $W_{in}$ is the transformation matrix, $b_{in}$ is the bias vector, and $\mathbf{x}_d$ is the input for LSTM. The hidden unit of LSTM has the same dimension as $\mathbf{x}_d$, \emph{i.e.} $\mathbf{h}^{l-1}\in\mathbb{R}^{\frac{C_2}{r}}$. The interaction between inner states and gates at $d$-th layer goes through the following steps:
\begin{equation}
\begin{aligned}
&\mathbf{i}_d=\sigma\left(W_i\cdot\left[\mathbf{h}_{d-1},\mathbf{x}_d\right] + b_i \right), \\
&\mathbf{f}_d=\sigma\left(W_f\cdot\left[\mathbf{h}_{d-1},\mathbf{x}_d\right] + b_f \right), \\
&\mathbf{g}_d={\rm tanh}\left(W_g\cdot\left[\mathbf{h}_{d-1},\mathbf{x}_d\right] + b_g \right), \\
&\mathbf{o}_d=\sigma\left(W_o\cdot\left[\mathbf{h}^{d-1},\mathbf{x}_d\right] + b_o \right) , \\
&\mathbf{c}_d=\mathbf{f}_d \odot {\mathbf{c}_{d-1}}+\mathbf{i}_d\odot\mathbf{g}_d, \\
&\mathbf{h}_d=\mathbf{o}_d \odot {\rm tanh}\left(\mathbf{c}_d\right), 
\end{aligned}
\label{lstm}
\end{equation}
where $\sigma$ refers to the sigmoid function and $\odot$ denotes element-wise multiplication. After the above processes, another fully connected layer transforms the hidden unit $\mathbf{h}_d$ into the step size vector $\Delta \mathbf{t}_d$:
\begin{equation}
\Delta \mathbf{t}_d = g_{out}(W_{out}\cdot \mathbf{h}_d + b_{out}),
\end{equation}
where $W_{out}\in\mathbb{R}^{C_2\times\frac{C_2}{r}}$, $b_{out}\in\mathbb{R}^{C_2}$, and $g_{out}$ denotes the non-linear sigmoid function that restricts the elements in the current step size between range $(0,1)$. The forward propagation in the current step is:
\begin{equation}
\mathbf{y}_{d+1}=\mathbf{y}_d+\mathcal{F}(\mathbf{y}_d,\mathbf{w}_d)\otimes\Delta \mathbf{t}_d,
\label{res_prop}
\end{equation}
where $\otimes$ represents the channel-wise multiplication. We have one controller for each size stage in ResNet, and all parameters above keep shared in the same size stage. 

We note that the channel-wise attention technique \cite{hu2017} has a similar formulation to Eq. (\ref{res_prop}). Our method differs from theirs in that our controller is data independent and does not rely on the feature space. What our time stepping aims to optimize is part of the structural information. When training finishes, our method discards the controller and has no considerable addition cost in inference phase (except a little calculation of multiplying step sizes). This cannot be realized by attention methods that are feature dependent. Besides, our experiments show that our method is compatible with the attention method.

\begin{table}[!t]
	\renewcommand\arraystretch{1.5}
	\caption{Parameter complexity introduced by the three methods in training and inference phases. $k_1$ and $k_2$ denote the kernel sizes of the convolution parameters.}
	\begin{center}
		\begin{tabular}{l|c|c}
			\hline
			Methods & Training & Inference \\ 
			\hline
			TSC$_{indp}$  & $\sum_{b=1}^{B}L_bC_b$ & $\sum_{b=1}^{B}L_bC_b$\\ 
			\hline
			TSC$_{2fc}$ & $\sum_{b=1}^{B}L_b\frac{C_b^{2}}{r}(1 + k_{1}k_{2})$ & $\sum_{b=1}^{B}L_bC_b$\\
			\hline
			TSC$_{LSTM}$ & $\sum_{b=1}^{B}\frac{C_b^{2}}{r}(1+ \frac{8}{r} + k_{1}k_{2})$ & $\sum_{b=1}^{B}L_bC_b$\\
			\hline
		\end{tabular}
	\end{center}
	\label{complex}
\end{table}

\subsection{Complexity Analysis}

We denote the \textbf{t}ime \textbf{s}tepping \textbf{c}ontroller using LSTM as TSC$_{LSTM}$. In addition to this structure, we also consider two other versions as its counterparts. We analyze their complexities in this subsection and compare their performance in experiments. The first one removes the LSTM layers Eq. (\ref{lstm}), and only keeps the input and output transformation layers but does not share their parameters. It is similar to the two fully-connected layers module in \cite{hu2017}. We denote this version as TSC$_{2fc}$. This structure is dependent on the convolution parameters but not aware of previous steps. The other one does not use a controller, and only introduces the step sizes $\{\Delta \textbf{t}_d\}^{D-1}_{d=0}$ as explicit parameters that are independent of weight matrices and previous steps. This version is denoted as TSC$_{indp}$.

When training finishes, only the optimized step sizes should be stored and the controller can be removed, which leads to little additional cost in inference phase. As for training complexity, TSC$_{indp}$ has the same cost since it does not have a controller. As for TSC$_{LSTM}$ and TSC$_{2fc}$, TSC$_{LSTM}$ consumes less parameters because of the weight-sharing in LSTM, while TSC$_{2fc}$ consumes less computation because it does not have the LSTM layers in Eq. (\ref{lstm}). Assuming that the network has $B$ feature size stages (as an example, $B=3$ for CIFAR and $4$ for ImageNet). There are $L_b$ layers in the $b$-th stage, and $C_b$ is the number of channels per layer. We compare the three methods' parameter complexity of training and inference phase in Table \ref{complex}. We will compare their performance and overhead in experiments.

\subsection{Implementations}

\begin{table}[!t]
	\small
	\renewcommand{\arraystretch}{1.3}
	\caption{Comparison between different controllers on ImageNet. They are all based on ResNet-50 structure. Our methods do not have much additional cost in inference phase.}
	\begin{center}
		\begin{tabular}{l|c|c|c}
			\hline
			\multirow{2}*{Methods} & top-1 & Params(M) & GFLOPs\\
			~ & \small (err.) & \small \emph{(train / infer)} & \small \emph{(train / infer)} \\
			\hline
			ResNet (reported) & $24.70$ & $25.56$ & $3.86$\\
			\hline
			ResNet (ours) & $24.42$ & $25.56$ & $3.86$\\
			\hline
			TSC$_{indp}$-ResNet & $24.12$ & $25.57 / 25.57$ & $3.86 / 3.86$\\
			\hline
			TSC$_{2fc}$-ResNet & $23.90$ & $31.58 / 25.57$ & $3.88 / 3.86$\\
			\hline
			TSC$_{LSTM}$-ResNet & $\textbf{{23.63}}$ & $27.83 / 25.57$ & $ 3.89 / 3.86$\\
			\hline
		\end{tabular}
	\end{center}
	\label{ablation1}
\end{table}

In experimental section, we test our proposed methods on the ResNet and its variants. Here we show the details of our implementation. For ResNets structures without bottleneck, there are two convolution layers in each residual block, and we have $k_1=k_2=3$ for the kernel sizes. We concatenate these two parts of projected parameters in layer $d$ by $\mathbf{\bar{w}}_d=[ \mathbf{\bar{w}}_d^{(1)}, \mathbf{\bar{w}}_d^{(2)} ]\in\mathbb{R}^{2k_1k_2C_2}$, where $C_2$ denotes the output channels in this layer. For ResNets structures with bottleneck, there are three convolution layers in each residual block, and the kernel size is $1$ for the first and third layers, and $3$ for the bottleneck layer. We only concatenate the first and third projected parameters by $\mathbf{\bar{w}}_d=[\mathbf{\bar{w}}_d^{(1)}, \mathbf{\bar{w}}_d^{(3)}]\in\mathbb{R}^{C_2+C_2/4}$, where $C_2/4$ is the number of bottleneck channels. We set the reduction $r$ to $4$ for ResNets without bottleneck, and $8$ for ResNets with bottleneck. An illustration of our method on ResNet-34 (without bottleneck) and ResNet-50 (with bottleneck) is shown in Appendix B.

\section{Experiments}

We conduct experiments on CIFAR-10, CIFAR-100, and ImageNet to validate our time stepping controller on ResNet and its variants. We also test the application of our method to two non-residual network structures.

\subsection{Datasets and Training Details}

\subsubsection{Datasets} For training sets of the ImageNet dataset, we adopt the standard data augmentation scheme \cite{he2016deep}. A $224\times 224$ crop is randomly sampled from the image or its horizontal flip. The input images are normalized by mean and standard deviation for each channel. All models are trained on the training set and we report the single center crop error rate on the validation set. For CIFAR-10 and CIFAR-100, we adopt a standard data augmentation scheme by padding the images 4 pixels filled with 0 on each side and then randomly sampling a $32\times32$ crop from each image or its horizontal flip. The images are normalized by mean and standard deviation.

\subsubsection{Training Details} We train our models using stochastic gradient descent (SGD) with the Nesterov momentum 0.9 and weight decay $10^{-4}$. The parameters are initialized following \cite{he2015delving}. For the ImageNet dataset, we train for 100 epochs with an initial learning rate of 0.1, and drop the learning rate every 30 epochs. A mini-batch has 256 images among 8 GPUs. For CIFAR-10 and CIFAR-100, we train for 300 epochs with a mini-batch of 64 images. The learning rate is set to 0.1 and divided by 10 at $50\%$ and $75\%$ of the training procedure. For our results on CIFAR, we run for 3 times with different seeds and report mean values.

\subsection{Ablation Study}

\begin{figure}[t]
	\begin{center}
		\centering
		\includegraphics[width=0.95\linewidth]{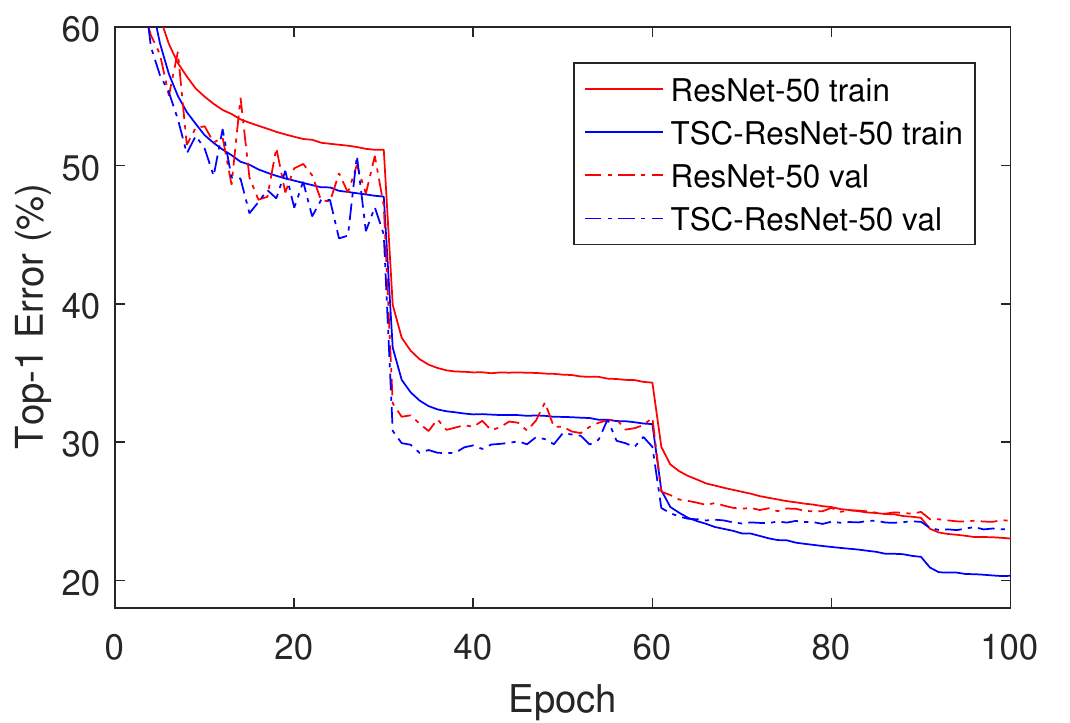}
	\end{center}
	\caption{Top-1 error rate training curves of the baseline ResNet-50 and our TSC$_{LSTM}$-ResNet-50 on ImageNet.}
	\label{fig3}
\end{figure}

In order to test the effectiveness of our proposed LSTM time stepping controller TSC$_{LSTM}$, we conduct experiments on ImageNet and compare with the two counterparts, TSC$_{indp}$ and TSC$_{2fc}$. We perform our methods with ResNet-50. As shown in Table \ref{ablation1}, our re-implementation has a slightly better performance than reported. When armed with TSC$_{indp}$, ResNet has a small performance improvement, due to the introduced step sizes as explicit parameters. It reveals that a trainable step size benefits the ResNet performance. TSC$_{2fc}$ and TSC$_{LSTM}$ have larger improvements, which shows that the controller dependent on the convolution parameters contributes to a better performance. TSC$_{LSTM}$ is further aware of previous steps, because of the memory brought by LSTM, and has an improvement of 0.79\% top-1 accuracy than baseline. The training curves of our TSC$_{LSTM}$ and baseline are compared in Figure \ref{fig3}. It is shown that our method has a superiority during the whole training procedure. The ablation study demonstrates the effectiveness of our design that the step sizes should be dependent on convolution parameters and aware of previous steps. We use the LSTM controller TSC$_{LSTM}$ for our later experiments.

Besides, the optimized step sizes belong to part of the structural information and are not data dependent. Thus, when training finishes, only the step sizes should be stored, and the additional overhead introduced by our controller can be spared in inference. As shown in Table \ref{ablation1}, compared with baseline, these three methods have little additional cost in inference. For training, TSC$_{indp}$ has the same parameters and computation as inference because it does not use a parameterized controller. TSC$_{LSTM}$ consumes less parameters but more computation than TSC$_{2fc}$ in training. It is consistent with our analysis in prior section.

\subsection{Improving the Performance}

\begin{table*}[!th]
	\renewcommand{\arraystretch}{1.3}
	
	\caption{Top-1 error rates on ImageNet with and without our proposed time stepping controller. The numbers in brackets show the performance improvement over the re-implemented baselines. The left and right numbers of Params or GFLOPs columns show the overhead of our methods during training and inference phase respectively. All baseline methods are not able to reduce the overhead during inference phase, so there is a single value on the corresponding columns.}
	\begin{center}
		\begin{tabular}{l|c|c|c|c|c|c}
			\hline
			\multirow{3}*{Models} & \multicolumn{3}{c|}{re-implementation} & \multicolumn{3}{c}{with our time stepping controller}\\
			\cline{2-7}
			~ & \multirow{2}*{Error. (\%)} & \multirow{2}*{Params(M)} & \multirow{2}*{GFLOPs} & Error. (\%) & Params(M) & GFLOPs\\
			~ & ~ & ~ & ~ & \small (gain) & \small \emph{(train / infer)} &  \small \emph{(train / infer)}\\
			\hline
			\hline
			ResNet-18 & 29.41 & 11.69 & 1.81 & 28.80 ${\textbf{(0.61)}}$ & 13.52 / \textbf{11.69} & 1.81 / \textbf{1.81} \\
			\hline
			ResNet-34 & 26.03 & 21.80 & 3.66 & 25.39 ${\textbf{(0.64)}}$ & 23.63 / \textbf{21.80} & 3.66 / \textbf{3.66} \\
			\hline
			ResNet-50 & 24.42 & 25.56 & 3.86 & 23.63 ${\textbf{(0.79)}}$ & 27.83 / \textbf{25.57} & 3.89 / \textbf{3.86} \\
			\hline
			ResNet-101 & 22.94 & 44.55 & 7.58 & 22.24 ${\textbf{(0.70)}}$ & 46.82 / \textbf{44.58} & 7.64 / \textbf{7.58} \\
			\hline
			\hline
			ResNeXt-50 & 22.84 & 25.03 & 3.77 & 22.23 ${\textbf{(0.61)}}$ & 27.30 / \textbf{25.04} & 3.80 / \textbf{3.77} \\
			\hline
			ResNeXt-101 & 21.88 & 44.18 & 7.51 & 21.17 ${\textbf{(0.71)}}$ & 46.45 / \textbf{44.21} & 7.57 / \textbf{7.51} \\
			\hline
			\hline
			SENet-50 & 23.27 & 28.09 & 3.87 & 22.75 ${\textbf{(0.52)}}$ & 30.36 / \textbf{28.10} & 3.90 / \textbf{3.87} \\
			\hline
			SENet-101 & 22.37 & 49.33 & 7.60 & 21.82 ${\textbf{(0.55)}}$ & 51.60 / \textbf{49.36} & 7.66 / \textbf{7.60} \\
			\hline
		\end{tabular}
	\end{center}
	\label{totaltab}
\end{table*}

We add our time stepping controller on ResNet families with different depths and different variants, including ResNeXt \cite{xie2016aggregated} and SENet \cite{hu2017}, to validate the ability of improving performance. As shown in Table \ref{totaltab}, for fair comparison, we re-implement the baseline methods and most of our re-implementation performance are superior to the reported numbers. 

When armed with our time stepping controller, it is shown that ResNets in different depths consistently have a 0.6\%-0.8\% improvement on performance. ResNet-50 has the largest accuracy gain. We see that our methods introduce a small number of parameters and computation for training, and nearly no considerable extra cost for inference. 

We also add our time stepping controller on ResNet variants to test the scalability. The implementation of ResNeXt is similar to ResNet. It is shown that our method is also effective to ResNeXts. For ResNeXt-50, it has a top-1 accuracy improvement of 0.61\%, while ResNeXt-101 has a larger gain of 0.71\%. 

Compared with feature operating modules, such as the channel-wise attention in SENet, our method may not have strong advantages for improving the performance, because the attention methods are data dependent, while ours are searching for adaptive step sizes, which are independent from features and belong to structural information. In spite of this, we note that TSC-ResNet-101 (22.24\% top-1 error rate) has surpassed the performance of SENet-101 (22.37\% top-1 error rate) using less parameters and computation. Our performance gain has little extra cost in inference, which cannot be realized by the feature dependent method SENet. We also show that our time stepping controller is compatible with SENets, even if they share a similar propagation formulation as Eq. (\ref{res_prop}). Our method reduces a top-1 error rate of 0.52\% for SENet-50, and 0.55 \% for SENet-101. 

From Table \ref{totaltab}, we observe that a deeper model benefits more from our time stepping controller in general. We believe that the reason is that a deeper network suffers more from the effects of step sizes. As indicated by Eq. (\ref{res_stab}), when depth increases, the cumulative influence of the spectral norm of weight matrices and step sizes become larger. In this case, an inappropriate step size would heavily impede the stability and performance of the network. It is in line with our intuition that deeper networks have more difficulties of training. Our method has an adaptive time stepping controller to adjust the steps sizes jointly with the network training, and thus helps more for deeper networks.

\begin{figure}[t]
	\centering
	\includegraphics[width=0.9\linewidth]{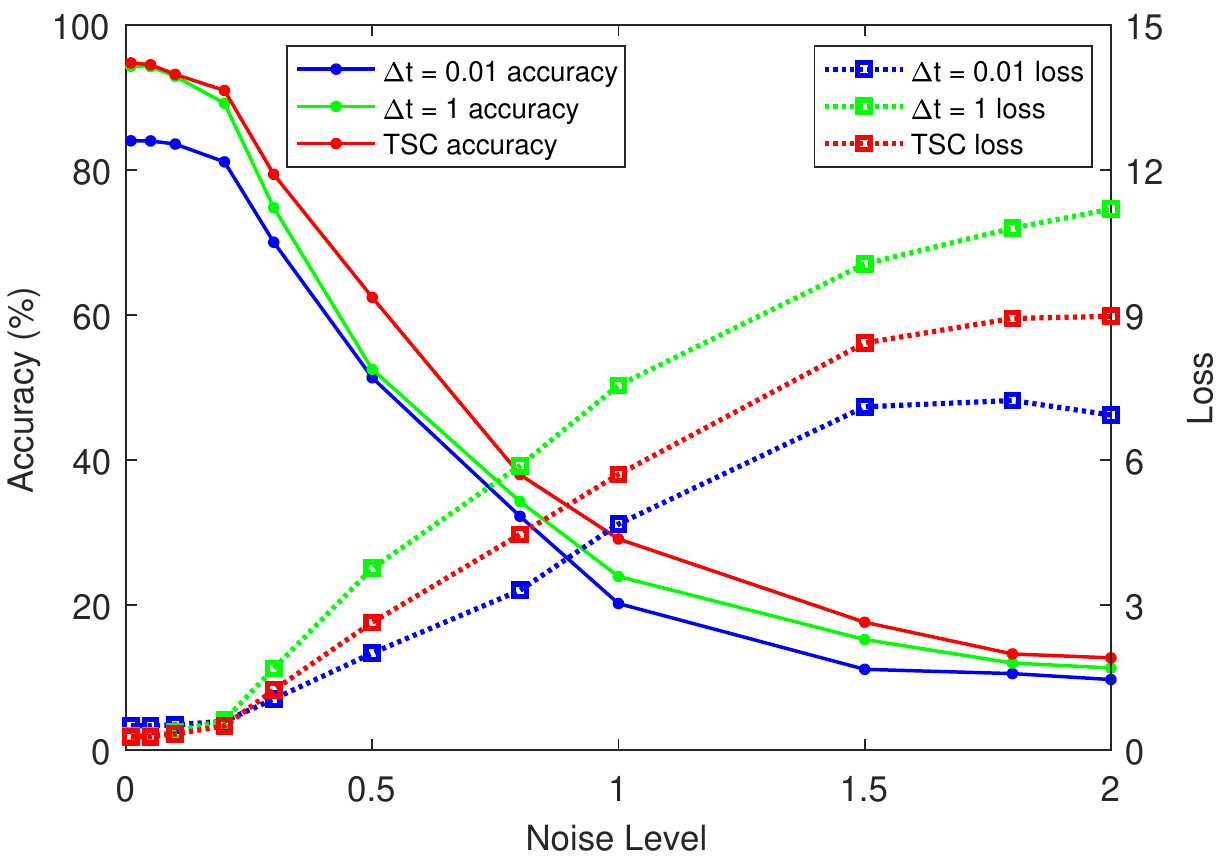}
	\caption{Under different noise level, the inference accuracy and loss of ResNets with step sizes 1, 0.01, and our time stepping controller (TSC) on CIFAR-10 test set.}
	\label{noise}
\end{figure}

\subsection{Improving the Stability}

\begin{figure}[t]
	\centering
	\includegraphics[width=0.75\linewidth]{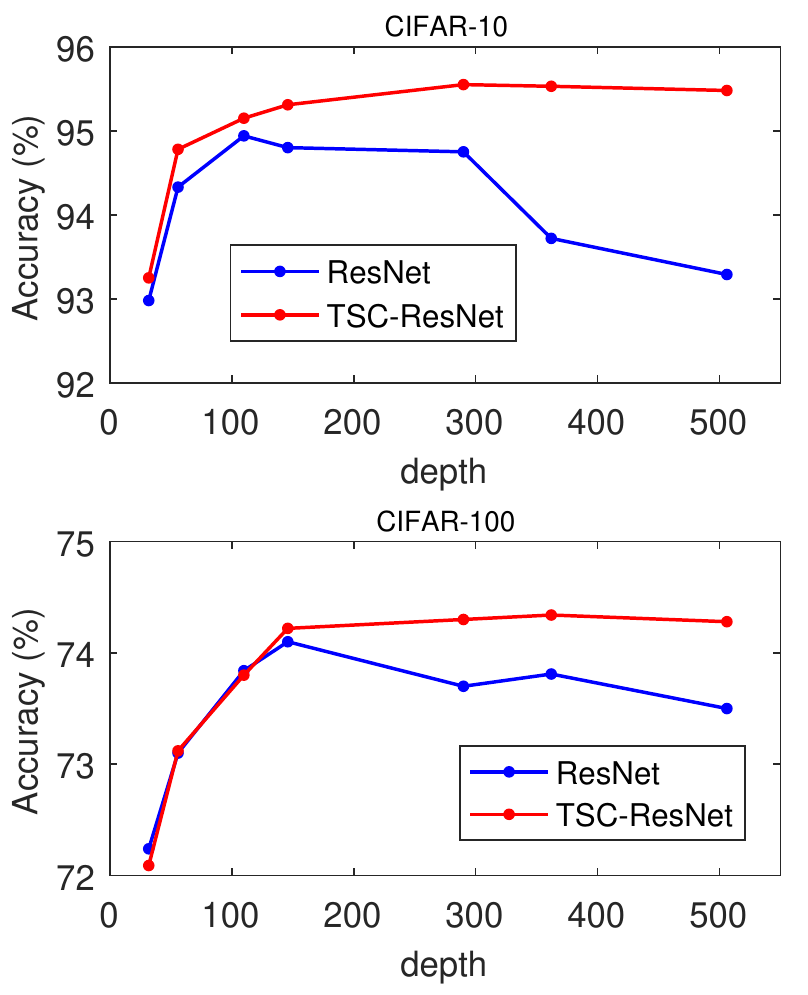}
	\caption{Comparison of ResNets and TSC-ResNets on CIFAR-10 and CIFAR-100 with increasing depths (32, 56, 110, 146, 290, 362 and 506 layers).}
	\label{increasing_depth}
\end{figure}

In order to test the ability of our time stepping controller to improve the stability, we conduct experiments to show the controller's robustness to perturbations and increasing depths. 

We train ResNet-56 on CIFAR-10 with different step sizes (0.01 and 1), and our time stepping controller. After training, we inject different level Gaussian noise to the input for inference on the test set. The level of perturbation is decided by the standard deviation of the synthetic Gaussian noise. As shown in Figure \ref{noise}, when the noise level increases, the accuracies of $\Delta t=0.01$ and $\Delta t=0.01$ both drop quickly. The loss of $\Delta t=1$ has a sharp increment compared with that of $\Delta t=0.01$. It is in line with our analysis in Eq. (\ref{res_stab}) that a small step size in each layer helps to bound the adverse effect caused by perturbations. However, the performance of $\Delta t=0.01$ is significantly worse than $\Delta t=1$. As a comparison, our time stepping controller offers adaptive step sizes. It is shown that TSC has a moderate loss increment with the noise level rising. Although TSC has a higher loss than $\Delta t=0.01$ when noise level is larger than 0.5, the accuracies of TSC are consistently better than $\Delta t=1$ and $\Delta t=0.01$. This demonstrates that our stepping controller improves the ResNet robustness to perturbations, and offers a good trade-off between the stability and performance.

We further test the robustness of our method to increasing depths. As shown in Figure \ref{increasing_depth}, we train ResNet and TSC-ResNet with different depths on the CIFAR-10 and CIFAR-100 datasets. For shallow residual networks, TSC-ResNet and ResNet have similar performances. When depth increases, the accuracies of ResNets approach a plateau, and then have a drop for very deep networks. A similar result is also reported in \cite{zhang2019towards}. As a comparison, the performance of TSC-ResNet is more stable. It keeps a slow increment in accuracy for large depths. We note that the performance gap between TSC-ResNet and ResNet is larger for deeper networks in general, which is consistent with our observation on the ImageNet experiments in the prior section. This demonstrates our analysis that deeper networks suffer from an accumulated instability, and thus gain more benefits from our adaptive time stepping controller.

\subsection{Analysis}

\begin{table}[!t]
	\renewcommand\arraystretch{1.1}
	\caption{The final optimized time step sizes (averaged along the channel dimension) of TSC-ResNet-50 and TSC-ResNet-101 in different layers. ``C-3\_1'' denotes the first residual block in the Conv-3 size stage.}
	\begin{center}
		\setlength{\tabcolsep}{4pt}
		\begin{tabular}{l|c|c|c|c|c|c}
			\hline
			Models & C-3\_1 & C-4\_1 & C-4\_2 & C-5\_1 & C-5\_2 & C-5\_3 \\
			\hline
			TSC-50 & 0.510 & 0.500 & 0.731 & 0.835 & 0.901 & 0.926 \\
			TSC-101 & 0.508 & 0.513 & 0.657 & 0.854 & 0.927 & 0.950 \\
			\hline
		\end{tabular}
	\end{center}
	\label{optimized}
\end{table}

As shown in Table \ref{optimized}, we average the optimized time step size vectors of TSC-ResNet-50 and TSC-ResNet-101 in different layers. We found that the step sizes in shallow layers of both TSC-ResNet-50 and TSC-ResNet-101 are centered around 0.5. For the layers in the last size stage (Conv-5\_1, 5\_2, 5\_3), the step sizes are approaching 1. Conv-3\_1 and 4\_1 keep the initial state and are mildly affected by our method, because each of our time stepping controller consider previous steps but they correspond to the first time step in each size stage. Besides, the shallow layers should have small step sizes to avoid accumulated instability. But for layers Conv-5\_1, 5\_2, and 5\_3, they should enlarge step sizes to achieve strong feature transformations for the final representation. 

We also observe that, deeper layers in Conv-5\_1, 5\_2, and 5\_3 converge to larger step sizes. For the same layers in Conv-5\_1, 5\_2, and 5\_3, TSC-ResNet-101 converges to a larger step size. It reveals that deeper layer or deeper model requires a larger step size.  It is in line with our experimental results that deeper models gain more benefits from our method in general. We believe that the reason lies in that deeper network corresponds to longer evolution time and suffer more from inappropriate time stepping. We also conjecture that our adaptive time stepping’s effects on shallow layers mainly ensure stability, while the adjustments for deep layers contribute to the performance gains.

\subsection{Extension to Non-residual Networks}
Our analyses and the development of our time stepping controller are based on the correspondence between residual networks and discrete dynamical systems. In order to test the scalability of our method to other networks, we add the controller to two non-residual network structures, DenseNet \cite{huang2016densely} and CliqueNet \cite{Yang_2018_CVPR}. The results are shown in Appendix C.

\section{Conclusion}

In this study, we use the correspondence between residual networks and discrete dynamical systems to unravel the physics of ResNets. We analyze the stability condition of the Euler method and ResNet propagation, and point out the effects of step sizes on the stability and performance of ResNets. Inspired by the adaptive time stepping in numerical methods of ODEs, we develop an adaptive time stepping controller that is dependent on the parameters of the network and aware of previous steps to adaptively adjust the step sizes and evolution time. Experiments on ImageNet, CIFAR-10, and CIFAR-100 show that our method is able to improve both performance and stability, without introducing much overhead in inference phase. Our method can also be applied to other non-residual network structures.

\section{Ackonwledgement}


Z. Lin is supported by NSF China (no.s 61625301 and 61731018), Major Scientific Research Project of Zhejiang Lab (no.s 2019KB0AC01 and 2019KB0AB02), and Beijing Academy of Artificial Intelligence. J. Wu is supported by the Fundamental Research Funds of Shandong University and SenseTime Research Fund for Young Scholars.

\bibliographystyle{aaai}\bibliography{tsc_aaai.bib}

\begin{thebibliography}{}

\bibitem[\protect\citeauthoryear{Ascher and Petzold}{1998}]{ascher1998computer}
Ascher, U.~M., and Petzold, L.~R.
\newblock 1998.
\newblock {\em Computer methods for ordinary differential equations and
  differential-algebraic equations}, volume~61.
\newblock Siam.

\bibitem[\protect\citeauthoryear{Baker \bgroup et al\mbox.\egroup
  }{2017}]{baker2016designing}
Baker, B.; Gupta, O.; Naik, N.; and Raskar, R.
\newblock 2017.
\newblock Designing neural network architectures using reinforcement learning.
\newblock In {\em ICLR}.

\bibitem[\protect\citeauthoryear{Behrmann, Duvenaud, and
  Jacobsen}{2019}]{behrmann2018invertible}
Behrmann, J.; Duvenaud, D.; and Jacobsen, J.-H.
\newblock 2019.
\newblock Invertible residual networks.
\newblock In {\em ICML}.

\bibitem[\protect\citeauthoryear{Chang \bgroup et al\mbox.\egroup
  }{2018a}]{chang2018reversible}
Chang, B.; Meng, L.; Haber, E.; Ruthotto, L.; Begert, D.; and Holtham, E.
\newblock 2018a.
\newblock Reversible architectures for arbitrarily deep residual neural
  networks.
\newblock In {\em AAAI}.

\bibitem[\protect\citeauthoryear{Chang \bgroup et al\mbox.\egroup
  }{2018b}]{chang2017multi}
Chang, B.; Meng, L.; Haber, E.; Tung, F.; and Begert, D.
\newblock 2018b.
\newblock Multi-level residual networks from dynamical systems view.
\newblock In {\em ICLR}.

\bibitem[\protect\citeauthoryear{Chen \bgroup et al\mbox.\egroup
  }{2017}]{chen2017deeplab}
Chen, L.-C.; Papandreou, G.; Kokkinos, I.; Murphy, K.; and Yuille, A.~L.
\newblock 2017.
\newblock Deeplab: Semantic image segmentation with deep convolutional nets,
  atrous convolution, and fully connected crfs.
\newblock {\em TPAMI} 40(4):834--848.

\bibitem[\protect\citeauthoryear{Chen \bgroup et al\mbox.\egroup
  }{2018}]{chen2018neural}
Chen, T.~Q.; Rubanova, Y.; Bettencourt, J.; and Duvenaud, D.~K.
\newblock 2018.
\newblock Neural ordinary differential equations.
\newblock In {\em NeurIPS},  6571--6583.

\bibitem[\protect\citeauthoryear{E}{2017}]{weinan2017proposal}
E, W.
\newblock 2017.
\newblock A proposal on machine learning via dynamical systems.
\newblock {\em Communications in Mathematics and Statistics}.

\bibitem[\protect\citeauthoryear{Haber and Ruthotto}{2017}]{haber2017stable}
Haber, E., and Ruthotto, L.
\newblock 2017.
\newblock Stable architectures for deep neural networks.
\newblock {\em Inverse Problems} 34(1):014004.

\bibitem[\protect\citeauthoryear{Haber \bgroup et al\mbox.\egroup
  }{2018}]{haber2018learning}
Haber, E.; Ruthotto, L.; Holtham, E.; and Jun, S.-H.
\newblock 2018.
\newblock Learning across scales---multiscale methods for convolution neural
  networks.
\newblock In {\em AAAI}.

\bibitem[\protect\citeauthoryear{Haber \bgroup et al\mbox.\egroup
  }{2019}]{eldad2019}
Haber, E.; Lensink, K.; Triester, E.; and Ruthotto, L.
\newblock 2019.
\newblock Imexnet {A} forward stable deep neural network.
\newblock In {\em ICML}.

\bibitem[\protect\citeauthoryear{Hairer, N{\o}rsett, and
  Wanner}{1991}]{hairer1991solving}
Hairer, E.; N{\o}rsett, S.~P.; and Wanner, G.
\newblock 1991.
\newblock {\em Solving ordinary differential equations. 1, Nonstiff problems}.
\newblock Springer-Vlg.

\bibitem[\protect\citeauthoryear{He \bgroup et al\mbox.\egroup
  }{2015}]{he2015delving}
He, K.; Zhang, X.; Ren, S.; and Sun, J.
\newblock 2015.
\newblock Delving deep into rectifiers: Surpassing human-level performance on
  imagenet classification.
\newblock In {\em ICCV},  1026--1034.

\bibitem[\protect\citeauthoryear{He \bgroup et al\mbox.\egroup
  }{2016}]{he2016deep}
He, K.; Zhang, X.; Ren, S.; and Sun, J.
\newblock 2016.
\newblock Deep residual learning for image recognition.
\newblock In {\em CVPR},  770--778.

\bibitem[\protect\citeauthoryear{He \bgroup et al\mbox.\egroup
  }{2019}]{he2019ode}
He, X.; Mo, Z.; Wang, P.; Liu, Y.; Yang, M.; and Cheng, J.
\newblock 2019.
\newblock Ode-inspired network design for single image super-resolution.
\newblock In {\em CVPR},  1732--1741.

\bibitem[\protect\citeauthoryear{Hu, Shen, and Sun}{2018}]{hu2017}
Hu, J.; Shen, L.; and Sun, G.
\newblock 2018.
\newblock Squeeze-and-excitation networks.
\newblock In {\em CVPR}.

\bibitem[\protect\citeauthoryear{Huang \bgroup et al\mbox.\egroup
  }{2017}]{huang2016densely}
Huang, G.; Liu, Z.; van~der Maaten, L.; and Weinberger, K.~Q.
\newblock 2017.
\newblock Densely connected convolutional networks.
\newblock In {\em CVPR}.

\bibitem[\protect\citeauthoryear{Jia \bgroup et al\mbox.\egroup
  }{2019}]{jia2019focnet}
Jia, X.; Liu, S.; Feng, X.; and Zhang, L.
\newblock 2019.
\newblock Focnet: A fractional optimal control network for image denoising.
\newblock In {\em CVPR},  6054--6063.

\bibitem[\protect\citeauthoryear{Kwakernaak and
  Sivan}{1972}]{kwakernaak1972linear}
Kwakernaak, H., and Sivan, R.
\newblock 1972.
\newblock {\em Linear optimal control systems}, volume~1.
\newblock Wiley-interscience New York.

\bibitem[\protect\citeauthoryear{Li and Hao}{2018}]{li2018icml}
Li, Q., and Hao, S.
\newblock 2018.
\newblock An optimal control approach to deep learning and applications to
  discrete-weight neural networks.
\newblock In {\em ICML}.

\bibitem[\protect\citeauthoryear{Li \bgroup et al\mbox.\egroup
  }{2017}]{li2017maximum}
Li, Q.; Chen, L.; Tai, C.; and Weinan, E.
\newblock 2017.
\newblock Maximum principle based algorithms for deep learning.
\newblock {\em The Journal of Machine Learning Research}.

\bibitem[\protect\citeauthoryear{Li \bgroup et al\mbox.\egroup
  }{2018}]{li2018optimization}
Li, H.; Yang, Y.; Chen, D.; and Lin, Z.
\newblock 2018.
\newblock Optimization algorithm inspired deep neural network structure design.
\newblock {\em arXiv preprint arXiv:1810.01638}.

\bibitem[\protect\citeauthoryear{Lu \bgroup et al\mbox.\egroup }{2018}]{Lu2018}
Lu, Y.; Zhong, A.; Li, Q.; and Dong, B.
\newblock 2018.
\newblock Beyond finite layer neural networks: Bridging deep architectures and
  numerical differential equations.
\newblock In {\em ICML}.

\bibitem[\protect\citeauthoryear{Miyato \bgroup et al\mbox.\egroup
  }{2018}]{miyato2018spectral}
Miyato, T.; Kataoka, T.; Koyama, M.; and Yoshida, Y.
\newblock 2018.
\newblock Spectral normalization for generative adversarial networks.
\newblock In {\em ICLR}.

\bibitem[\protect\citeauthoryear{Ruthotto and Haber}{2018}]{ruthotto2018deep}
Ruthotto, L., and Haber, E.
\newblock 2018.
\newblock Deep neural networks motivated by partial differential equations.
\newblock {\em arXiv preprint arXiv:1804.04272}.

\bibitem[\protect\citeauthoryear{Simonyan and
  Zisserman}{2015}]{simonyan2014very}
Simonyan, K., and Zisserman, A.
\newblock 2015.
\newblock Very deep convolutional networks for large-scale image recognition.
\newblock In {\em ICLR}.

\bibitem[\protect\citeauthoryear{Sokoli{\'c} \bgroup et al\mbox.\egroup
  }{2017}]{sokolic2017robust}
Sokoli{\'c}, J.; Giryes, R.; Sapiro, G.; and Rodrigues, M.~R.
\newblock 2017.
\newblock Robust large margin deep neural networks.
\newblock {\em IEEE Transactions on Signal Processing} 65(16):4265--4280.

\bibitem[\protect\citeauthoryear{Szegedy \bgroup et al\mbox.\egroup
  }{2015}]{szegedy2015going}
Szegedy, C.; Liu, W.; Jia, Y.; Sermanet, P.; Reed, S.; Anguelov, D.; Erhan, D.;
  Vanhoucke, V.; and Rabinovich, A.
\newblock 2015.
\newblock Going deeper with convolutions.
\newblock In {\em CVPR},  1--9.

\bibitem[\protect\citeauthoryear{Tran \bgroup et al\mbox.\egroup
  }{2018}]{tran2018closer}
Tran, D.; Wang, H.; Torresani, L.; Ray, J.; LeCun, Y.; and Paluri, M.
\newblock 2018.
\newblock A closer look at spatiotemporal convolutions for action recognition.
\newblock In {\em CVPR},  6450--6459.

\bibitem[\protect\citeauthoryear{Veit, Wilber, and
  Belongie}{2016}]{veit2016residual}
Veit, A.; Wilber, M.~J.; and Belongie, S.
\newblock 2016.
\newblock Residual networks behave like ensembles of relatively shallow
  networks.
\newblock In {\em NIPS},  550--558.

\bibitem[\protect\citeauthoryear{Xie \bgroup et al\mbox.\egroup
  }{2017}]{xie2016aggregated}
Xie, S.; Girshick, R.; Doll{\'a}r, P.; Tu, Z.; and He, K.
\newblock 2017.
\newblock Aggregated residual transformations for deep neural networks.
\newblock In {\em CVPR}.

\bibitem[\protect\citeauthoryear{Yang \bgroup et al\mbox.\egroup
  }{2018}]{Yang_2018_CVPR}
Yang, Y.; Zhong, Z.; Shen, T.; and Lin, Z.
\newblock 2018.
\newblock Convolutional neural networks with alternately updated clique.
\newblock In {\em CVPR}.

\bibitem[\protect\citeauthoryear{Yoshida and
  Miyato}{2017}]{yoshida2017spectral}
Yoshida, Y., and Miyato, T.
\newblock 2017.
\newblock Spectral norm regularization for improving the generalizability of
  deep learning.
\newblock {\em arXiv preprint arXiv:1705.10941}.

\bibitem[\protect\citeauthoryear{Zhang and Laura}{2018}]{zhang2018smooth}
Zhang, J., and Laura, W.
\newblock 2018.
\newblock Smooth inter-layer propagation of stabilized neural networks for
  classification.
\newblock {\em arXiv preprint arXiv:1809.10315}.

\bibitem[\protect\citeauthoryear{Zhang \bgroup et al\mbox.\egroup
  }{2019}]{zhang2019towards}
Zhang, J.; Han, B.; Wynter, L.; Low, K.~H.; and Kankanhalli, M.
\newblock 2019.
\newblock Towards robust resnet: A small step but a giant leap.
\newblock In {\em IJCAI}.

\end{thebibliography}

\section{Supplementary Material}

\subsection{Appendix A: Proof of Proposition 1}

\begin{prop}
	Consider a ResNet with $D$ residual blocks and variable step sizes $\Delta t_j$ for each step. Let $\epsilon$ be the perturbation coming from noise or adversary and satisfies $||\mathbf{y}^{\epsilon}_0-\mathbf{y}_0||=\epsilon$. We have:
	\begin{equation}
	||\mathbf{y}^{\epsilon}_D-\mathbf{y}_D||\le \epsilon \cdot \prod_{j=0}^{D-1}(1+||\mathbf{W}_j||_2\Delta t_j),
	\label{res_stab_supp}
	\end{equation}
	where $||\mathbf{W}_j||_2$ denotes the spectral norm of weight matrix in each residual block.
\end{prop}

\begin{proof}
	We have the propagation of ResNet in the $d$-th layer as:
	\begin{equation}
	\mathbf{y}_{d+1}=\mathbf{y}_d+\mathcal{F}(\mathbf{y}_d, \mathbf{W}_d)\Delta t_d.
	\end{equation}
	Denote the perturbation in the $d+1$-th layer as $\epsilon_{d+1}=||\mathbf{y}^{\epsilon}_{d+1}-\mathbf{y}_{d+1}||$. Then we have,
	\begin{align}
	\epsilon_{d+1} & =||\mathbf{y}^{\epsilon}_{d}-\mathbf{y}_{d}+(\mathcal{F}(\mathbf{y}^{\epsilon}_d, \mathbf{W}_d)-\mathcal{F}(\mathbf{y}_d, \mathbf{W}_d))\Delta t_d|| \notag\\
	& \leq \epsilon_{d} + ||\mathcal{F}(\mathbf{y}^{\epsilon}_d, \mathbf{W}_d)-\mathcal{F}(\mathbf{y}_d, \mathbf{W}_d)||\Delta t_d.
	\label{epsi}
	\end{align}
	
	We simplify the residual branch in the $d$-th layer as a composite function composed of the linear operator $\mathbf{W}_d$, and the ReLU non-linear activation $\mathbf{A}$, which is a diagonal matrix. The value in $\mathbf{A}$ equals to one if the corresponding element in $\mathbf{y}_d$ is positive, otherwise equals to zero. As pointed out by \cite{yoshida2017spectral}, the non-linearity in neural networks usually comes from piecewise linear functions, such as ReLU, maxpooling, etc. Because the perturbation $\epsilon_d$ is small, $\mathbf{y}^{\epsilon}_d$ can be considered as a neighborhood of $\mathbf{y}_d$, and the residual function behaves as a linear operator near $\mathbf{y}_d$, then we have:
	\begin{align}
	& \frac{||\mathcal{F}(\mathbf{y}^{\epsilon}_d, \mathbf{W}_d)-\mathcal{F}(\mathbf{y}_d, \mathbf{W}_d)||}{||\mathbf{y}^{\epsilon}_{d}-\mathbf{y}_{d}||} =\frac{||\mathbf{A}\mathbf{W}_d(\mathbf{y}^{\epsilon}_d-\mathbf{y}_d)||}{||\mathbf{y}^{\epsilon}_{d}-\mathbf{y}_{d}||} \notag\\
	& \leq ||\mathbf{A}\mathbf{W}_d||_2 \leq ||\mathbf{A}||_2||\mathbf{W}_d||_2 \leq ||\mathbf{W}_d||_2
	\label{aw}
	\end{align}
	where $||\cdot||_2$ denotes the matrix spectral norm defined as:
	\begin{equation}
	||X||_2=\max_{u\in\mathbb{R}^n,u\ne 0}\frac{||Xu||_2}{||u||_2},
	\end{equation}
	which corresponds to the largest singular value of $X$. So we have $||\mathbf{A}||_2\le1$, and Eq. (\ref{aw}) holds. 
	
	Consequently, it follows from Eq. (\ref{epsi}) and Eq. (\ref{aw}) that:
	\begin{equation}
	\epsilon_{d+1}\leq\epsilon_{d}(1+||\mathbf{W}_d||_2\Delta t_d).
	\label{eq6}
	\end{equation}
	Apply Eq. (\ref{eq6}) to all layers in the network, then we have Eq. (\ref{res_stab_supp}), which concludes the proof.
	$\hfill\blacksquare$  
\end{proof}

\newpage
\subsection{Appendix B: TSC-ResNet Structure on ImageNet}

We show the structures of TSC-ResNet-34 (without bottleneck) and TSC-ResNet-50 (with bottleneck) in Table \ref{structure} in the next page.

\begin{table}[!t]
	\renewcommand\arraystretch{1.4}
	\caption{Error rates on the CIFAR-10 and CIFAR-100 datasets to test the application of our time stepping controller to two non-residual network structures.}
	\begin{center}
		\begin{tabular}{l|c|c}
			\hline
			Methods & CIFAR-10 & CIFAR-100 \\ 
			\hline
			DenseNet  & 5.24 & 24.42\\ 
			\hline
			TSC-DenseNet & 5.10 & 23.35 \\
			\hline
			CliqueNet & 5.03 & 22.80\\
			\hline
			TSC-CliqueNet & 4.79 & 22.04\\
			\hline
		\end{tabular}
	\end{center}
	\label{non-residual}
\end{table}

\begin{table*}[ht]
	\renewcommand{\arraystretch}{1.1}
	\caption{A diagram of the structures TSC-ResNet-34 and TSC-ResNet-50. The fc layers in each block are newly introduced by our method. The first and last row in the fc layers refer to the input and output transformation. The middle row indicates the layers in Eq. (11) of our paper.}
	\begin{center}
		\begin{tabular}{c|c|c}
			\hline
			Output size & TSC-ResNet-34 & TSC-ResNet-50\\
			\hline
			112$\times$ 112 & \multicolumn{2}{c}{conv, 7$\times$ 7,stride 2}\\
			\hline
			56$\times$56 & \multicolumn{2}{c}{max pool, 3$\times$3, stride 2}\\
			\hline
			56$\times$56 & conv-$ \left(
			\begin{matrix}
			3\times3, 64 \\
			3\times3, 64 \\
			\end{matrix}  \right)\times 3 $,
			\ fc-$\left(
			\begin{matrix}
			\left[3\times3\times128,16\right]\times 1 \\
			\left[32,16\right]\times 4 \\
			\left[16,64\right]\times 1
			\end{matrix}
			\right)
			$
			& conv-$\left(
			\begin{matrix}
			1\times1, 64 \\
			3\times3, 64 \\
			1\times1, 256
			\end{matrix}\right)\times 3 $,
			\ fc-$\left(
			\begin{matrix}
			\left[320,32\right]\times 1\\
			\left[64,32\right]\times 4\\
			\left[32,256\right]\times 1
			\end{matrix}
			\right)
			$\\
			\hline
			28$\times$28 & conv-$ \left(
			\begin{matrix}
			3\times3, 128 \\
			3\times3, 128 \\
			\end{matrix}  \right)\times 4 $,
			\ fc-$\left(
			\begin{matrix}
			\left[3\times3\times256,32\right]\times 1 \\
			\left[64,32\right]\times 4 \\
			\left[32,128\right]\times 1
			\end{matrix}
			\right)
			$
			& conv-$\left(
			\begin{matrix}
			1\times1, 128 \\
			3\times3, 128 \\
			1\times1, 512
			\end{matrix}\right)\times 4 $,
			\ fc-$\left(
			\begin{matrix}
			\left[640,64\right]\times 1\\
			\left[128,64\right]\times 4\\
			\left[64,512\right]\times 1
			\end{matrix}
			\right)
			$\\
			\hline
			14$\times$14 & conv-$ \left(
			\begin{matrix}
			3\times3, 256 \\
			3\times3, 256 \\
			\end{matrix}  \right)\times 6 $,
			\ fc-$\left(
			\begin{matrix}
			\left[3\times3\times512,64\right]\times 1 \\
			\left[128,64\right]\times 4 \\
			\left[64,256\right]\times 1
			\end{matrix}
			\right)
			$
			& conv-$\left(
			\begin{matrix}
			1\times1, 256 \\
			3\times3, 256 \\
			1\times1, 1024
			\end{matrix} \right)\times 6 $,
			\ fc-$\left(
			\begin{matrix}
			\left[1280,128\right]\times 1\\
			\left[256,128\right]\times 4\\
			\left[128,1024\right]\times 1
			\end{matrix}
			\right)
			$\\
			\hline
			7$\times$7 & conv-$ \left(
			\begin{matrix}
			3\times3, 512 \\
			3\times3, 512 \\
			\end{matrix}  \right)\times 3 $,
			\ fc-$\left(
			\begin{matrix}
			\left[3\times3\times1024,128\right]\times 1 \\
			\left[256,128\right]\times 4 \\
			\left[128,512\right]\times 1
			\end{matrix}
			\right)
			$
			& conv-$ \left(
			\begin{matrix}
			1\times1, 512 \\
			3\times3, 512 \\
			1\times1, 2048
			\end{matrix}\right)\times 3 $,
			\ fc-$\left(
			\begin{matrix}
			\left[2560,256\right]\times 1\\
			\left[512,256\right]\times 4\\
			\left[256,2048\right]\times 1
			\end{matrix}
			\right)
			$\\
			\hline
			1$\times$1 & \multicolumn{2}{c}{global average pool, fc-1000, softmax}\\
			\hline
		\end{tabular}
	\end{center}
	\label{structure}
\end{table*}

\subsection{Appendix C: Extension to Non-residual Networks}

Our analyses and the development of our time stepping controller are based on the correspondence between residual networks and discrete dynamical systems. In order to test the scalability of our method to other networks, we add the controller to two non-residual network structures, DenseNet \cite{huang2016densely} and CliqueNet \cite{Yang_2018_CVPR}. Although the two network structures do not directly have the identity mapping as a short cut path, they introduce the dense connection, by which the addition in ResNet propagation is replaced with concatenation. Similarly, we use the parameters that generate the new feature at each step as the input of our controller. The new feature is channel-wise multiplied with the output step size, and then concatenated with old features to form the current layer. For DenseNet, we use the setting of $k=12$ and $d=40$, where $k$ is the number of channels for each new feature, and $d$ denotes the network depth. For CliqueNet, we use their setting of $k=36$ and $T=12$, where $k$ has the same meaning as DenseNet, and $T$ is the number of layers in the network. 

As shown in Table \ref{non-residual}, when armed with our time stepping controller, both DenseNet and CliqueNet have performance improvements on the CIFAR-10 and CIFAR-100 datasets. Although the two networks do not have the identity mapping to be closely connected with the discretization of dynamical systems, our experiments indicate that our method is scalable and can be applied to other non-residual networks.

\end{document}